\documentclass{IEEEtran}

\newcommand{\vect}[1]{\mathbf{#1}}

\newcommand{\diffs}[3]{\frac{\partial^2 #1}{
\ifx#2#3 
\partial #2^2
\else
\partial #2 \partial #3
\fi
}}

\makeatletter

\let\proof\@undefined
\let\endproof\@undefined
\makeatother

\usepackage{amssymb}  
\usepackage{amsthm}
\usepackage{mathtools}
\usepackage{amsmath, amsfonts, epsfig, xspace, bm}
\usepackage{soul}

\usepackage{tikz}
\usetikzlibrary{arrows,matrix,positioning}
\usepackage{mdframed}
\usepackage{empheq}

\usepackage{tikz}

\usepackage[export]{adjustbox}

\newtheorem{prop}{Proposition}

\newtheorem{ass}{Assumption}

\newtheorem{thm}{Theorem}

\theoremstyle{definition}

\newtheorem{remark}{Remark}

\usepackage{verbatim}

\usepackage{color}
\usepackage{subfig}
\usepackage{graphicx}
\usepackage{caption}

\usepackage{mathptmx}
\usepackage{times}

\usepackage{algorithm}
\usepackage[noend]{algpseudocode}

\makeatletter
\def\BState{\State\hskip-\ALG@thistlm}
\makeatother

\algnewcommand{\Inputs}[1]{%
	\State \textbf{Inputs:}
	\Statex \hspace*{\algorithmicindent}\parbox[t]{.8\linewidth}{\raggedright #1}
}
\algnewcommand{\Initialization}[1]{%
	\State \textbf{Initialization:}
	\Statex \hspace*{\algorithmicindent}\parbox[t]{.8\linewidth}{\raggedright #1}
}

\usepackage[us]{datetime}
\usepackage{caption}
\newcommand\blue[1]{{\textcolor{blue}{#1}}}

\newcommand{\colvec}[2][.9]{%
	\scalebox{#1}{%
		\renewcommand{\arraystretch}{1}%
		$\begin{bmatrix}#2\end{bmatrix}$%
	}
}

\usepackage{hyperref} 
\hypersetup{
	colorlinks,
	citecolor=black,
	filecolor=black,
	linkcolor=black,
	urlcolor=black,
	pdfauthor={},
	pdfsubject={},
	pdftitle={}
}

\usepackage{todonotes}

\usepackage{graphicx}

\usepackage{subfig}

\usepackage{latexsym}
\usepackage{color}

\usepackage{cite}

\usepackage{caption}

\graphicspath{{images/}}

\newif\ifdraft
\draftfalse

\usepackage[normalem]{ulem}

\author{Enrica Rossi$^{1}$, Marco Tognon$^2$, Ruggero Carli$^1$, Antonio Franchi$^{3,4}$, Luca Schenato$^1$,
	\thanks{\hspace{-1.5em}$^1$Department of Information Engineering, University of Padova, Italy {\tt \scriptsize \href{mailto:enrica.rossi.1@studenti.unipd.it}{enrica.rossi.1@studenti.unipd.it},
			\href{mailto:carlirug@dei.unipd.it}{carlirug@dei.unipd.it},
			\href{mailto:schenato@dei.unipd.it}{schenato@dei.unipd.it}}}
	\thanks{\hspace{-1.5em}$^2$Autonomous Systems Lab, ETH Zurich, 8092 Z\"urich, Switzerland, {\tt \scriptsize  \href{mailto:mtognon@ethz.ch}{mtognon@ethz.ch}}}		
	\thanks{\hspace{-1.5em}$^3$Robotics and Mechatronics lab, University of Twente, Enschede, The Netherlands, {\tt \scriptsize\href{mailto:a.franchi@utwente.nl}{a.franchi@utwente.nl}}}
	\thanks{\hspace{-1.5em}$^4$LAAS-CNRS, Universit\'e de Toulouse, CNRS, Toulouse, France, {\tt \scriptsize  \href{mailto:marco.tognon@laas.fr}{marco.tognon@laas.fr}}}  	
	}

\usepackage{paralist}

\begin{document}
	\title{\bf {Control of over-redundant cooperative manipulation via sampled communication}}
	
	\maketitle

	\begin{abstract} 
	In this work we consider the problem of mobile robots that need to manipulate/transport an object via cables or robotic arms. We consider the scenario where the number of manipulating robots is redundant, i.e. a desired object configuration can be obtained by different configurations of the robots. The objective of this work is to show that communication can be used to implement cooperative local feedback controllers in the robots to improve disturbance rejection and reduce structural stress in the object. In particular we consider the realistic scenario where measurements are sampled and transmitted over wireless, and the sampling period is comparable with the system dynamics time constants. We first propose a kinematic model which is consistent with the overall systems dynamics under high-gain control and then we provide sufficient conditions for the exponential stability and monotonic decrease of the configuration error under different norms. Finally, we test the proposed controllers on the full dynamical systems showing the benefit of local communication. 
	\end{abstract}
	\vspace{0cm}
	\section{INTRODUCTION}\label{sec:intro}
Object manipulation and transportation via the adoption of multiple mobile autonomous vehicles is becoming a very active research area \cite{Lippi:18,feng2020overview} thanks to the potential benefits in building constructions, infrastructure maintenance, and product delivery \cite{KondakAJ10}. These vehicles are connected to the manipulated object via robotics arms, cables or joints and by properly controlling the vehicle position and orientation it is possible to indirectly control the position and orientation of the object. The objective of the aforementioned application is typically to control the vehicles in order to have the object to reach a desired configuration or to follow a certain trajectory within a cluttered environment. 

While object manipulation by multiple robots and specifically parallel manipulators is a well studied and established subject under the assumption of reliable and centralized computation\cite{Marino:17}, in the presence of mobility with wireless communication then not all ``end-effectors'' can be reliably controlled by a single intelligence since communication is delayed or even lost. Cooperative manipulation via multiple mobile robots then become more challenging. One possible direction to address cooperative manipulation is to avoid communication altogether and rely on \emph{implicit communication}. 
It has been shown, that in the context of load transportation and manipulation, by measuring the forces experienced by the end-effector of  the robot can indirectly provide information on how to create a leader-follower control strategy to manipulate the object \cite{Marino:18}. For example if a robot (leader) pulls the object in a certain direction, then the other robots (followers) will be able to estimate such direction and can therefore accompany the object along the desired direction~\cite{SiegwartIJRR19,2019l-GabTogPalFra}.
Although very effective, such approach suffers for the need of force sensors  at each end-effector or precise wrench estimation methods, it exerts major stresses on the object, and manipulation speed is reduced since the control is delayed till substantial forces are measured. 
Moreover, there are relevant scenarios, such as load transportation with unmanned aerial vehicles (UAVs) where the object is typically manipulated via flexible cables, in which force sensors cannot be easily employed.
An alternative approach to overcome the aforementioned problems is to employ wireless communication such as Wi-Fi which has high-bit rates and is a readily available technology. 
Wi-Fi potentially allows for communication rates (packets per second) of 1-kHz \cite{Branz:20} but in multi-agent scenario and in order to guarantee limited packet losses, it is reasonable to consider rates of 10-100Hz or even smaller. At these rates, the typical assumption of considering continuous-time feedback control, i.e. measurements continuously fed to the controller, is unrealistic and sampled dynamics has to be considered in order to design control strategies that are still effective. A step in this direction has been presented in \cite{Rossi:20}, where sampled communication has been considered explicitly in a kinematic model for aerial load transportation with three drones each connected with two cables to support a load. The vehicles are assumed to be able to directly control their velocity vector and the objective is to move the system configuration (load pose as well the cables inclination) from one (feasible) initial configuration into another (feasible) desired configuration.   

In this work we extend \cite{Rossi:20} along several direction. First we consider a more general scenario where multiple robots are intended to manipulate the same object. This scenario is relevant for example in load transportation where the load weight is large and multiple drones are needed to sustain it. However, differently from  \cite{Rossi:20}, the system becomes over-redundant, i.e. there are multiple end-effectors/vehicles configuration which lead to the same load configuration, thus implying that not all end-effectors velocities are feasible. We address this problem by proposing a more general kinematic model that is shown to be consistent with the overall system dynamics if high-gain force  control in employed and separation of time-scales applies. Such kinematic model is used to design two control strategies, one off-line and the other on-line similarly to \cite{Rossi:20}, which are guaranteed to provide exponential convergence under certain sufficient conditions that involve the sampling period and the size of the initial conditions set. The second novelty is to guarantee, besides exponential stability, also smooth transition from the initial configuration to the final configuration in terms of monotonically decreasing error for different types of norms. This is important to guarantee that this transition does not approach singular configurations. Finally, despite the proposed controllers are designed based on a kinematic model, we show how they can be extended also to a more realistic dynamical model and we tested them via extensive numerical simulations.

The paper is organized as follows. In Section \ref{sec:modeling} we describe the kinematic model of the over-redundant system of interest. In particular we show that it is consistent with the system dynamics when a high-gain force control is applied. Moreover we formulate the problem we aim at solving, that is, designing a feedback controller driving the system toward a desired final configuration.
In Section \ref{sec:cont_time_contr} we propose a feedback controller when continuous-time measurements are available. In Section \ref{sec:discr_time_contr} we properly re-design the controller in case of sampled measurements focusing on two strategies, one off-line and one on-line. In Section \ref{sec:simul}, we provide numerical simulations testing the proposed controllers. Finally, in Section \ref{sec:concl}, we gather our conclusions.


	\section{MODELING AND PROBLEM FORMULATION}
	\label{sec:modeling}
	\subsection{Kinematics of multi-agent systems}
	\label{subsec:kinematics}
	
	We consider a general multi-agent systems composed by $N$ agents which are connected through a common object $L$. This is the case for example in aerial load transportation where there are $N$ aerial vehicles connected though flexible cables to a common load $L$. We assume that the system can be described by $N$ Lagrangian variables each associated with an agent and possibly of different dimension, i.e. $\vect{q}_i\in\mathbb{R}^{m_i}$, and an additional Lagrangian variable associated with the object $\vect{q}_L\in\mathbb{R}^{m_L}$. In the aerial load transportation examples, the former Lagrangian variables correspond to angles of the cables with respect to the load, and the latter Lagrangian variable to the load pose. We finally assume that system exhibits a \emph{star-like interaction topology}, where each agent is connected only to the common object, i.e. the positions of the $N$ agents, $\vect{p}_i$ can be expressed in terms of the Lagrangian variables as follows:

		\begin{equation}
	\vect{p}_i  = 
		\vect{h}^{(i)}(\vect{q}_{i}, \vect{q}_L) 	
		\label{eq:h(q)}
	\end{equation}
	where, for $i=1,\ldots,N$, $\vect{p}_i \in \mathbb{R}^{n_i}$ denotes the position of the $i$-th agent. From now on, we denote by $\vect{p}$ the vector collecting the agents configurations, i.e., $\vect{p}= 	\colvec{	\vect{p}^\top_1 \;\, \dots \;\, \vect{p}^\top_N}^\top \, \in \, \mathbb{R}^n$ where $n=\sum_{i=1}^N n_i$, by $\vect{q}$ the vector collecting the Lagrangian coordinates of the system, i.e., $\vect{q = \colvec{\vect{q}_1^\top \;\, \dots \;\, \vect{q}_N^\top \; \vect{q}^\top_L}^\top}\,\in \, \mathbb{R}^m$ where $m=\sum_{i=1}^N m_i \,+ \, m_L$, and by $\vect{h}$ the overall kinematic map, i.e., $\vect{h = \colvec{\vect{h}_1(\cdot)^\top \;\, \dots \;\, \vect{h}_N(\cdot)^\top}^\top}$. Based on this assumption, the  differential kinematic model that can be derived differentiating \eqref{eq:h(q)}:
		\begin{equation}
	\dot{\vect{p}}=\vect{A_q}\dot{\vect{q}}
	\label{eq:kin_model}
	\end{equation}
	where the Jacobian $ \vect{A_q} =\frac{\partial \vect{h(q)}}{\partial \vect{q}} \,\,\in \mathbb{R}^{n \times m}$ has the following form
		\begin{align}
	\vect{A_q}  =\colvec{
		\begin{array}{ccc|c}
		\vect{A}^{(1)}_{\vect{q}_{1}} &  & \text{\fontsize{5mmm}{5mm}\selectfont$\vect{0}$} &  \vect{A}^{(1)}_{\vect{q}_{L}}\\
		& \ddots & &  \vdots\\
		\text{\fontsize{5mmm}{5mm}\selectfont$\vect{0}$} & & \vect{A}^{(N)}_{\vect{q}_{N}} &   \vect{A}^{(N)}_{\vect{q}_{L}}
		\end{array}
	}\in \mathbb{R}^{n \times m}\label{eq:jacobian1}
	\end{align}
	and
	$
	\vect{A}^{(i)}_{\vect{q}_{i}} =   \frac{\partial \vect{h}^{(i)}(\vect{q})}{\partial \vect{q}_i} \in \mathbb{R}^{n_i \times m_i}$ and  $\vect{A}^{(i)}_{\vect{q}_{L}} =   \frac{\partial \vect{h}^{(i)}(\vect{q})}{\partial \vect{q}_L} \in \mathbb{R}^{n_i \times m_L}.
	$


	\begin{figure}
		\centering
		\includegraphics[width=0.7\columnwidth]{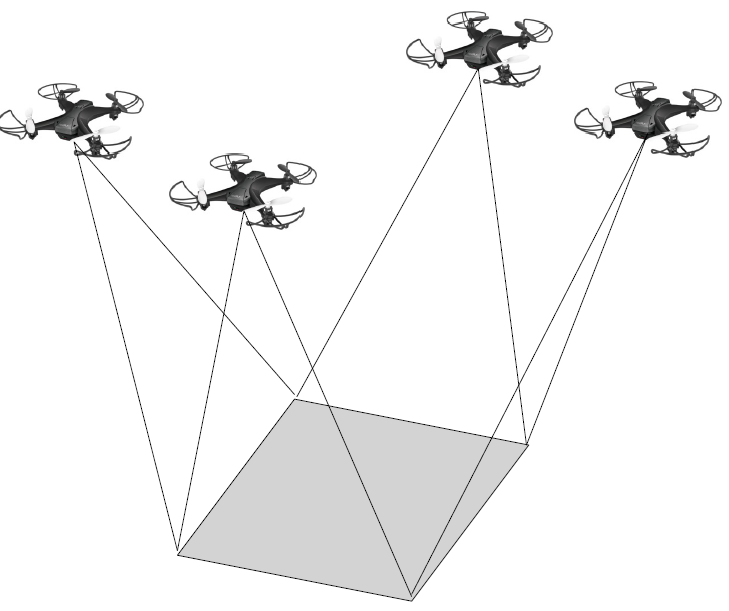}
		\caption{Example of a non-square multi-agent system where four aerial robots transport a payload: $n=12$ (positions $\mathbb{R}^3$ of four UAVs ) and $m=10$ (position $\mathbb{R}^3$ and orientation of the load $\mathbb{R}^3$, and four cable-load angle $\mathbb{R}^4$).}
		\label{fig:squared_flycrane}
	\end{figure}


It is  also useful to introduce
the manifold \mbox{$\mathcal{M}:=\{ ( \vect{p}, \vect{q})\, |\, \vect{p}=\vect{h}(\vect{q})\}$} and its tangent space at a point $(\vect{p}, \vect{q})$ as \mbox{$T_{(\vect{p}, \vect{q})}\mathcal{M} := \{ ( \dot{\vect{p}}, \dot{\vect{q}})\, |\, \dot{\vect{p}}=\frac{\partial \vect{h}}{\partial \vect{q}}\dot{\vect{q}}\}$}. Under the assumption that $\frac{\partial \vect{h}}{\partial \vect{q}}=\vect{A_q}\in\mathbb{R}^{n\times m}$ is full rank then a possible parametrization of such space is given by
$$ T_{(\vect{p}, \vect{q})}\mathcal{M}  =  \{ (\vect{A_q} \dot{\vect{q}} , \dot{\vect{q}}), \forall \dot{\vect{q}}\in \mathbb{R}^m\}$$
i.e., the tangent space is a $m$-dimensional linear space embedded into an $(n+m)$-dimensional space. 

\begin{ass}\label{ass:fullRankSystem}
We consider the over-redunt scenario where $n>m$, i.e. $\vect{A_q}$ is a tall-matrix.
\end{ass}


\subsection{System dynamics under high-gain force feedback}

To provide a more realistic behaviour of the multi-agent robotic systems, it is important to consider the associated dynamical model which can be written without loss of generality as
\begin{equation}
\vect{M(q)\ddot{q}+C(q,\dot{q})\dot{q}+g(q)=J^\top\!\!(q)F}
\label{eq:joint_dyn_eq}
\end{equation}
where $\vect{M(q)}$ is the joint-space inertia matrix, $\vect{C(q,\dot{q})}$ is the Coriolis and centripetal coupling matrix, $\vect{g(q)}$ is the gravity term, and
$\vect{J^\top(q)F}$ describes the effect of the external forces applied to the system, being $\vect{J(q)}= \vect{A_q}$ the Jacobian matrix and $\vect{F}=[\vect{f}^\top_1 \cdots \vect{f}^\top_N]^\top$, where $\vect{f}_i\in \mathbb{R}^{n_i}$ is the vector force applied by each agent. For the aerial load transportation example we refer to \cite{Sreenath-RSS-13} for the detailed derivation of the dynamical model. For sake of simplicity, we assume that each of the $N$ agent can directly control its force vector $\vect{f}_i$, i.e. the global force vector $\vect{F}$ is fully controllable. Without loss of generality, we can neglect $  \vect{g(q)}$ since a gravity compensation action can be implemented by the controller. Hence \eqref{eq:joint_dyn_eq} becomes
\begin{equation}
\vect{M(q)\ddot{q}+C(q,\dot{q})\dot{q}}=\vect{A}_\vect{q}^\top \vect{F}.
\label{eq:simplified_dyn}
\end{equation}
Let us assume to apply an external force $\vect{F}$ to the robot  of the type
\begin{equation}
\vect{F}:=-\alpha(\vect{\dot{p}-u})
\label{eq:Fdef}
\end{equation}
where $\alpha>0$ is a scalar gain, $\dot{\vect{p}}$ is the velocity vector of the $N$ agents and $\vect{u}\in \mathbb{R}^n$ is the vector of desired velocities. Such feedback choice has the objective to virtually increase viscous damping so that the inertia in the dynamics becomes negligible. 

Recalling the kinematic constraint \eqref{eq:kin_model}, the dynamical model can be rewritten as
\begin{align*}
\vect{M(q)\ddot{q}+C(q,\dot{q})\dot{q}}
=-\alpha\vect{A_q}^\top \vect{A_q} \vect{\dot{q}}- \alpha  \vect{\vect{A}_\vect{q}^\top u
}
\end{align*}
or equivalently 
\begin{equation}
\epsilon \vect{\ddot{q}}=\vect{M(q)}^{-1}\big(-\vect{A}_\vect{q}^\top \vect{A}_\vect{q} \vect{\dot{q}} + \vect{A}_\vect{q}^\top \vect{u}- \epsilon \vect{C(q,\dot{q})\dot{q}} \big)
\label{eq:dyn_model}
\end{equation}
where $\epsilon =1/\alpha$. Observe that for high gain $\alpha\gg 1$, i.e. for $\epsilon\approx 0 $, the dynamical equations are satisfied only if 
\begin{equation}\label{eqn:HG}
\alpha \gg 1 \Longrightarrow -\vect{A_q}^\top \vect{A}_\vect{q} \vect{\dot{q}} + \vect{A_q^\top u} \approx 0,
\end{equation}
If $\vect{A_q}$ is full-rank, then the previous expression is equivalent to 
\begin{equation}\label{eqn:kin_q}
\vect{\dot{q}} \approx (\vect{A_q}^\top \vect{A_q})^{-1} \mathbf{A}^\top_{\mathbf{q}}   \vect{u}=\mathbf{A}_{\mathbf{q}}^\dagger   \vect{u},
\end{equation}
where $\mathbf{A}_{\mathbf{q}}^\dagger \in \mathbb{R}^{m\times n}$ is the pseudo-inverse of $\mathbf{A}_{\mathbf{q}}$.
Recalling the kinematic constraint \eqref{eq:kin_model}, then the previous expression also implies
\begin{equation}\label{eqn:kin_u}
\vect{\dot{p}} \approx  \mathbf{A}_{\mathbf{q}}\mathbf{A}_{\mathbf{q}}^\dagger   \vect{u} = \Pi_{\mathbf{q}}  \vect{u}=:\vect{u}_{\mathbf{q}},
\end{equation}
where $\Pi_{\mathbf{q}(t)}=\mathbf{A}_{\mathbf{q}} \mathbf{A}_{\mathbf{q}}^\dagger \in \mathbb{R}^{n \times n}$ is the orthogonal projection matrix which projects the desired velocity vector $\vect{u}\in\mathbb{R}^n$, onto a $m$-dimensional subspace that corresponds to the column-span of the matrix $ \mathbf{A}_{\mathbf{q}}$, for which a pictorial representation is given in Fig.~\ref{fig:manifold_proj_u}. 
\begin{figure}
	\centering
	\includegraphics[width=0.6\columnwidth]{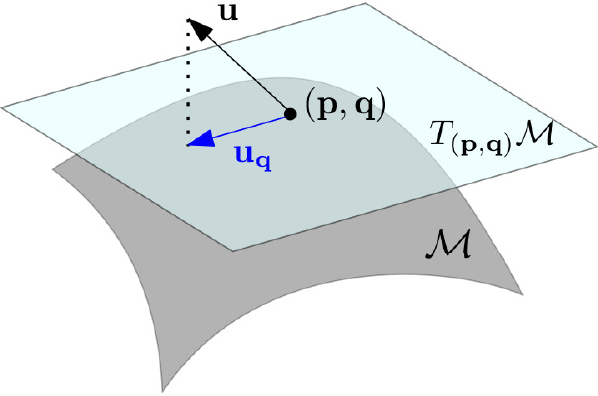}
	\caption{Manifold of structural constraints and tangent space for projection.}
	\label{fig:manifold_proj_u}
\end{figure}

A number of important considerations are in order: the first is that approximation \eqref{eqn:HG} is true only if certain conditions hold, and it corresponds to the slow-dynamics manifold of the global dynamical system~\eqref{eq:joint_dyn_eq}  \cite{khalil2002nonlinear}. We will show in Section~\ref{sec:cont_time_contr} that such conditions hold for our problem formulation and we refer the interested reader to \cite{khalil2002nonlinear} for general theory of singular perturbations and separation of time-scale principle. The second consideration is that approximation~\eqref{eqn:HG} leads to simpler controllable kinematic model for the Lagrangian variables given by \eqref{eqn:kin_q} which will be use to devise a kinematic feedback controller $\mathbf{u} = \kappa(\mathbf{q})$ that will allow us to drive the Lagrangian variables $\mathbf{q}$ to a desired configuration $\mathbf{q}^r$. The third consideration is that the actual agent velocity vector $ \vect{\dot{p}}$ cannot track any desired velocity vector $\vect{u}$, but they will follow its projection $\vect{u}_{\mathbf{q}}$ on the feasible tangent space $T_{(\vect{p}, \vect{q})}\mathcal{M}$. 

	\subsection{Problem Formulation}\label{subsec:problem_form}
	 
As mentioned above, the goal of this paper is to design 
a feedback controller in order to steer the system from an initial configuration $\mathbf{q}^0$ to a desired final configuration $\mathbf{q}^r$, where the latter could be a way point of a trajectory that is generated offline. To do so, we will consider the simplified kinematic model suggested in~\eqref{eqn:kin_q} that we report here for convenience:
\begin{equation}\label{eq:kinematic_model}
\vect{\dot{q}}(t)=\mathbf{A}_{\mathbf{q}(t)}^\dagger   \vect{u}(t)
\end{equation}
which generalizes to over-redundant system $m>n$ the kinematic model proposed in \cite{Rossi:20} given by $\vect{\dot{q}}(t)=\mathbf{A}_{\mathbf{q}(t)}^{-1}   \vect{u}(t) $ which was limited to square systems where $m=n$.

Similarly to \cite{Rossi:20}, we consider two possible architectures
	\vspace{0cm}
	\begin{eqnarray}
	\mathbf{u}_i(t) &=& \kappa_i^s (\vect{q}_i(t),\vect{q}_L(t);\vect{q}^r_i,\vect{q}^r_L), \ \ i=1,\ldots, N \label{eqn:sparse}\\
	\mathbf{u}_i(t) &=& \kappa_i^c (\vect{q}(t);\vect{q}^r) , \hspace{1.7cm} i=1,\ldots, N \label{eqn:centralized}
	\end{eqnarray}
	where the former tries to maintain the same sparsity of the Jacobian \ifdraft
	\blue{\st{-- referred to as \emph{sparse feedback} --}}
	\else 
	\fi while the latter exploit the full knowledge of the vector $\vect{q}$.  
	As in \cite{Rossi:20}, we will show that, if $\vect{q}(t)$ is continuously accessible to the local planner, then \eqref{eqn:sparse} is sufficient to drive the system to $\vect{q}^r$; otherwise, if $\vect{q}(t)$ is available at sampled time instants, then the two architectures give rise to two different strategies with
different performance and computational requirements. Eventually, the final controller which will be used in the dynamical model will be:
$$ F = -\alpha(\vect{\dot{p}}(t)-\kappa(\mathbf{q}(t))).$$

In order to guarantee that $\mathbf{A}_{\mathbf{q}(t)}$ is full-rank at all times we need to guarantee that $\vect{q}(t)$ remains in a feasible set. Differently from \cite{Rossi:20} which considered only ball sets based on the 2-norm, here we consider more general norms that might be more convenient when Lagrangian variables are of different nature (angles vs lengths) or we need to weight more some specific ones as compared to others. 	
		 More specifically, for $R>0$, we define the following sets
		 \begin{align*}
		 \mathcal{B}_{2}(\vect{q}^r)&=\{ \vect{q}\in\mathbb{R}^m\, | \, \| \vect{q} - \vect{q}^r\|_{\bm{D},2}< R\}\\
		 \mathcal{B}_{\infty}(\vect{q}^r)&=\{ \vect{q}\in\mathbb{R}^m\, | \, \| \vect{q} - \vect{q}^r\|_{\bm{D},\infty}<R \}		 
		 \end{align*} 
	where $\bm{D}$ is a weighting matrix such that $\bm{D} = \text{diag}\left\{\bm{D}_1, \ldots, \bm{D}_N, \bm{D}_L\right\}$, with $\bm{D}_i = \frac{1}{d_i^2} \vect{I}_{m_i}$ and $\bm{D}_L =\frac{1}{d_L^2} \vect{I}_{m_L}$, and where	 
	\begin{align*}
	\| \vect{q} - \vect{q}^r\|_{\bm{D},2}&= \| \bm{D} \left(\vect{q} - \vect{q}^r \right)\|_{2} =  \sqrt{\sum_i \frac{\|q_i-q_i^r\|^2}{d_i^2}}\\
	\| \vect{q} - \vect{q}^r\|_{\bm{D},\infty}&= \| \bm{D} \left(\vect{q} - \vect{q}^r \right)\|_{\infty} =\max_i \frac{\|q_i-q_i^r\|_\infty}{d_i^2}
	\end{align*} 
	 are respectively the $2$-norm and $\infty$-norm, weighted with $\bm{D}$. Without loss of generality we make the following assumption on the weights of matrix $\bm{D}$,
	 \begin{equation}\label{eq:d_i}
	 \min \left\{d_1,\ldots, d_N, d_L \right\} \geq 1.
	 \end{equation}
	 
	 From now on, for the sake of notational convenience, by the symbol $ \mathcal{B}_{*}(\vect{q}^r)$ we denote interchangeably either the set $ \mathcal{B}_{2}(\vect{q}^r)$ or the set $\mathcal{B}_{\infty}(\vect{q}^r)$.
	 
	 We conclude this section with the following assumption which is necessary to avoid singular kinematic configurations:
	\begin{ass}\label{ass:fullRankSystem}
		The following relations hold:
		\begin{enumerate}
			\item the matrix $\vect{A_q}$ is full-rank and the map $\vect{h}$ is three times continuously differentiable for all $q\in\mathcal{B}_{*}(\vect{q}^r)$. In addition, these properties can be extended by continuity on the closure of such set, defined as $\overline{\mathcal{B}}_{*}(\vect{q}^r)$;
			\item $\vect{q}^0\in \mathcal{B}_{*}(\vect{q}^r)$.
		\end{enumerate}
	\end{ass}

	\section{Feedback controller: continous-time }\label{sec:cont_time_contr}
	
	In this section we describe the feedback controller that generates the velocities vector $\vect{u}$ when continuous-time measurements are available. In such scenario
	we propose the following control law
		\begin{align}
		\mathbf{u}(t) =  -\mathbf{A}_{\mathbf{q}(t)}\,\mathbf{K}\,(\mathbf{q}(t)-\mathbf{q}^r)
		\label{eq:ctrl}
		\end{align}
		where $\mathbf{K}$ is a gain matrix to be designed. According to \eqref{eq:kinematic_model} and \eqref{eq:ctrl}, we have that the dynamics of $\vect{q}(t)$ evolves as		
		\begin{align}\label{eq:controlled_kin}
		\dot{\vect{q}}(t) =- \mathbf{A}_{\mathbf{q}(t)}^\dagger\mathbf{A}_{\mathbf{q}(t)}\mathbf{K}(\mathbf{q}(t)-\mathbf{q}^r)=-\mathbf{K}(\mathbf{q}(t)-\mathbf{q}^r). 
	\end{align}
where we exploit the fact $\mathbf{A}_{\mathbf{q}(t)}^\dagger\mathbf{A}_{\mathbf{q}(t)}=\vect{I}_m$ since $\mathbf{A}_{\mathbf{q}(t)}$ is assumed to be full column-rank.
In this paper, we assume that the gain matrix has the form $\vect{K} = k \cdot \overline{\vect{K}}$, where $k$ is a positive scalar gain, i.e., $k>0$, and where $\overline{\vect{K}}$ is a diagonal matrix with the following structure
$$
\overline{\vect{K}} = \text{diag}\left\{\vect{K}_1, \ldots, \vect{K}_N, \vect{K}_L\right\},
$$ 
where $\vect{K}_i = k_i \vect{I}_{m_i}$, $\vect{K}_L =k_L \vect{I}_{m_L}$ being also $k_1,\ldots,k_N, k_L$ positive real numbers, i.e., $k_1,\ldots,k_N, k_L \,>\,0$. Notice that a different weight is assigned to each $\vect{q}_i$, thus allowing for different convergence rates among the different variables.
	
The convergence properties of control law \eqref{eq:ctrl} are established in the next Proposition.

	\begin{prop}\label{prop:convergence_q}
	Consider system described by \eqref{eq:kinematic_model} and \eqref{eq:ctrl}, where $\vect{K}$ is defined as above. Assume Assumption \ref{ass:fullRankSystem} holds true. Then, for any $(N+2)$-upla of positive real numbers $k, k_1,\ldots,k_N,k_L$, we have that 
	\begin{enumerate}
		\item for all $t \geq 0$, $\vect{q}(t) \in \mathcal{B}_{*}(\vect{q}^r)$;
		\item the trajectory $\vect{q}(t)$ converges exponentially fast to $\vect{q}^r$.
	\end{enumerate}
\end{prop}
\begin{proof}
Let $\vect{e}=\vect{q}-\vect{q}^r$. Then from \eqref{eq:controlled_kin} we have that $\dot{\vect{e}}=-\vect{K} \vect{e}$ whose solution is $\vect{e}(t) = e^{-\vect{K}t} \vect{e}(0)$ where $e^{-\vect{K}t}$ is a diagonal matrix whose diagonal elements are exponentially decaying functions. This concludes the proof. 
\end{proof}
It is worth stressing two important properties of the proposed control strategy. Firstly, the trajectory $\vect{q}(t)$ generated by \eqref{eq:ctrl} never exits the ball $\mathcal{B}_{*}(\vect{q}^r)$, thus avoiding any singular point. Secondly, feedback law \eqref{eq:ctrl} can be rewritten component-wise as 
$$
\vect{u}_i =\vect{A}^{(i)}_{\vect{q}_{i}}  \left(\vect{q}_i - \vect{q}_i^r\right)   + \vect{A}^{(i)}_{\vect{q}_{L}} \left(\vect{q}_L - \vect{q}_L^r\right),
$$
that is, \eqref{eq:ctrl} exhibits the decentralized structure of control architecture in \eqref{eqn:sparse}.

We conclude this section by showing that the exponential stability property established in the previous Proposition 
holds true also for the dynamic model in \eqref{eq:simplified_dyn} provided the gain $\alpha$ is sufficiently high. This fact is formally stated in the next theorem.

\begin{thm}
	\label{thm:sing_pert_sys_adapted}
	Consider the dynamic model \eqref{eq:simplified_dyn} with $\vect{F}$ given in \eqref{eq:Fdef} and $\vect{u}$ defined as in \eqref{eq:ctrl}. 
	There exist $\overline{\alpha}>0$ and a neighborhood of $\vect{q}^r$, say $\bar{\mathcal{B}}(\vect{q}^r)$, contained in $\mathcal{B}_{*}(\vect{q}^r)$, i.e., $\bar{\mathcal{B}}(\vect{q}^r) \subseteq \mathcal{B}_{*}(\vect{q}^r)$, such that, if $\vect{q}(0)\in \bar{\mathcal{B}}(\vect{q}^r)$ and $\alpha>\overline \alpha$, then
		\begin{enumerate}
		\item for all $t \geq 0$, $\vect{q}(t) \in \bar{\mathcal{B}}(\vect{q}^r)$;
		\item the trajectory $\vect{q}(t)$ converges exponentially fast to $\vect{q}^r$.
	\end{enumerate}
	\end{thm}
	\begin{proof}
			The proof follows from a direct application of Theorem 11.4 of \cite{khalil2002nonlinear}. 		
	In order to apply this Theorem we first rewrite the dynamics of the system in a suitable state-space form; to do that let us define the state variables $\vect{x}=\vect{q}-\vect{q}^r$ and $\vect{z}=\dot{\vect{q}}$.
By simple algebraic manipulations it is possible to see that  \eqref{eq:dyn_model} when \eqref{eq:ctrl} is applied can be rewritten in state-space form as
	\begin{align}	\label{eq:state_space_form_cont_time_separation}
	\hspace{-4mm}\begin{array}{lll}
	\dot{\vect{x}}\!\!\! &=& \!\!\!\vect{z}\\
	\varepsilon \dot{\vect{z}} \!\!\!&=& \!\!\! -\vect{M_{x}}^{\!-1} \vect{A_x}^\top \vect{A_x}\vect{z} \!-\! k \vect{M_x}^{\!-1} \vect{A_x}^\top \vect{A_x} \mathbf{x} \!-\!\varepsilon \vect{M_x}^{\!-1} \vect{C_{x,z}} \, \, \vect{z}
	\end{array} 
	\end{align}
	where $\vect{A_{\vect{x}}}, \, \, \vect{M_{x}}, \, \, \vect{C_{x,z}}$ are compact notations for $\vect{A_{\vect{x}+\vect{q}^r}}, \, \, \vect{M_{\vect{x}+\vect{q}^r}}, \, \, \vect{C_{x+q^r,z}}$
	
	System in \eqref{eq:state_space_form_cont_time_separation} can be analyzed in a multitime-scale scenario \cite{khalil2002nonlinear}, where the dynamical system is characterized by slow and fast transients as far as the response to external stimuli is concerned. In particular, \eqref{eq:state_space_form_cont_time_separation} can be written as
	\begin{align*}
	\left\{  \begin{array}{ll}
	\dot{\vect{x}} = \vect{f}(\vect{z})\\
	\varepsilon \dot{\vect{z}}= \vect{g}(\vect{x},\vect{z},\varepsilon)
	\end{array} \right.
	\end{align*}
	that is, in the form introduced in Theorem 11.4 of \cite{khalil2002nonlinear}, except for the fact that some dependences do not appear (the function $\vect{f}$ depends only on $\vect{z}$ and does not depend on $(t,\vect{x,z})$). 
	To make the proof as self-contained as possible, we have reported the statement of Theorem 11.4 in Appendix \ref{app:stand_sing_pert_model}. Next, we show that the five hypotheses of Theorem 11.4 are satisfied by system in \eqref{eq:state_space_form_cont_time_separation}. 	

\noindent
{\bf{Assumption 1}}: In our scenario it is easy to verify that $\vect{f(0)=0}$ and $\vect{g(0,0,\varepsilon)}=\vect{0}$. 

\noindent
{\bf{Assumption 2}}: In our scenario we have the equation $\vect{0=g(e_x,z},0)=-\vect{M_\vect{x}}^{-1}[(\vect{A}_{\vect{x}}^\top \vect{A_{\vect{x}}})\vect{z}+ k  (\vect{A_{\vect{x}}}^\top \vect{A_{\vect{x}}})\vect{x}]$. Since $\vect{A_{\vect{x}}}^\top\vect{A_{\vect{x}}}$ and $\vect{M_\vect{x}}$ are strictly p.d. matrices by assumption, then the equation is equivalent to $\vect{z}+k\vect{x}=\vect{0} \Leftrightarrow \vect{z}=-k \vect{x}=\vect{h(x)}$ that is an isolated root. Moreover note that $\vect{h(0)=0}$. 

\noindent
{\bf{Assumption 3}}: The validity of this condition is guaranteed from Ass. \ref{ass:fullRankSystem}.

\noindent
{\bf{Assumption 4}}: In our scenario the reduced system is
	$$\vect{\dot{x}}=\vect{f(h(x))}=-k\vect{x}$$
	that is exponentially stable since $k>0$.
	
\noindent
{\bf{Assumption 5}}:	In our scenario the boundary-layer system is
	$$\frac{\text{d}\vect{y}}{\text{d}\tau}=\vect{g}(\vect{x},\vect{y+h(x),0})$$
	where 
	\begin{align*}
	\vect{g}(\vect{x},\vect{y+h(x),0})&=\vect{g}(\vect{x},\vect{y}-k \vect{x},0)\\
	&=-\vect{M_{\vect{x}}}^{-1}(\vect{A}_{\vect{x}}^\top \vect{A_{\vect{x}}})[(\vect{y}-k \vect{x})+ k  \vect{x}]\\
	&=-\vect{M_{\vect{x}}}^{-1}(\vect{A}_{\vect{x}}^\top \vect{A_{\vect{x}}})\vect{y}
	\end{align*}
	From now on, to keep the notation lighter, we omit the subscripts of the matrices $\vect{M}$ and $\vect{A}$.
Note that $\vect{M}$ is positive definite, hence $\vect{M^{-\frac{1}{2}}=(M^{-\frac{1}{2}})^\top}$ and
	\begin{align*}
	\vect{M^{-1}A^\top A}&=\vect{M^{-\frac{1}{2}}M^{-\frac{1}{2}}A^\top A}=\vect{M^{-\frac{1}{2}}(M^{-\frac{1}{2}}\vect{A^\top A}M^{-\frac{1}{2}})M^{\frac{1}{2}}}\\
	&=\vect{M^{-\frac{1}{2}}(A M^{-\frac{1}{2}})^\top(A M^{-\frac{1}{2}})M^{\frac{1}{2}}}	\end{align*}
	Thus $\vect{M^{-1}A^\top A}$ is similar to $(\vect{A M}^{-\frac{1}{2}})^\top(\vect{A M}^{-\frac{1}{2}})$ and, in turn, they have the same eigenvalues. Since $(\vect{A M}^{-\frac{1}{2}})^\top(\vect{A M}^{-\frac{1}{2}})$ is definite positive, then the eigenvalues of $\vect{M^{-1}A^\top A}$ are positive real numbers thus implying that the matrix $-\vect{M_{\vect{x}}}^{-1}(\vect{A}_{\vect{x}}^\top \vect{A_{\vect{x}}})$ is Hurwitz . It follows that the boundary-layer system is exponentially stable. Since $\mathcal{B}_*(\vect{q}^r)$ is compact, the exponential stability property is uniform in $\vect{x}$.

	The fact that the assumptions of Theorem \ref{thm:Khalil} are satisfied implies the existence of a Lyapunov function $V(\vect{x})$ for the reduced system and a Lyapunov function $W(\vect{x}, \vect{z})$ for the boundary system. Consider the function
	$$
	\nu(\vect{x}, \vect{z}) = V(\vect{x}) + W(\vect{x}, \vect{z}).
	$$
	Then, exploiting the properties of $\vect{f}$ and $\vect{g}$, and following the lines of proof of Theorem 11.4 of \cite{khalil2002nonlinear}, we can conclude that there exists $\bar{\epsilon}>0$ and a neighborhood of $\vect{q}^r$, $\bar{\mathcal{B}}(\vect{q}^r) \subseteq \mathcal{B}_{*}(\vect{q}^r)$, such that for any $0<\epsilon < \bar{\epsilon}$ and $\vect{q}(0) \in \bar{\mathcal{B}}(\vect{q}^r)$ we have that $\vect{q}(t) \in \mathcal{B}_{*}(\vect{q}^r)$ for all $t>0$ and $\dot{\nu} <0$ for all $t>0$. This concludes the proof.
	\end{proof}
	
\section{Feedback controller: Sampled Measurements}\label{sec:discr_time_contr}
		
We now consider the evolution of \eqref{eq:kinematic_model} under sampled dynamics, that is we assume that $\vect{q}$ is measured at the time instants $hT$, $h=0, 1, 2, \ldots$ where $T$ is the sampling time. Moreover, we assume that the vehicles reference velocity $\vect{u}(t)$ is kept constant within a time window $T$ using \eqref{eq:ctrl}:
\begin{align}\label{eq:SampledControlLaw}
\vect{u}(t) &=\vect{u}_h= - k\vect{A}_{\vect{q}_h} \overline{\vect{K}} (\vect{q}_h-\vect{q}^r),   \,\,\,\,\,\, hT  \leq t <(h+1)T,
\end{align}
for $h=1,2, \ldots$, and where we defined $\vect{q}_h = \vect{q}(hT)$.
In this scenario, assuming $\vect{A}_{\vect{q}(t)}, \ t\geq 0$ full rank, the evolution of $\vect{q}(t)$ becomes	
\begin{align}\label{eq:sampled_dyn}
\dot{\mathbf{q}}(t) &= -k\vect{A}^{\dagger}_{\vect{q}(t)}\vect{A}_{\vect{q}_h} \overline{\vect{K}} (\vect{q}_h-\vect{q}^r), \,\,\,\, hT  \leq t <(h+1)T.
\end{align}

In this section we assume that $\overline{\vect{K}}$ is assigned a-priori (or, equivalently, that the values of $k_L$ and $k_i$, $i=1,\ldots,N$, are assigned a-priori), while $k$ is a positive parameter to be designed, possibly time varying, i.e., $k=k_h$, such that the stability of the system is still guaranteed. 

The approach we propose for the design of $k$ within each interval $\left[hT, \,(h+1)T\right]$, is similar to the one proposed in \cite{Rossi:20} for square systems and 
it is based on the analysis of the following auxiliary  system  
\begin{align}\label{eqn:NL}
\dot{\mathbf{e}}'  (\tau;\vect{e}_h) &=- A^{\dagger}_{\vect{q}^r+\vect{e}'(\tau;\vect{e}_h)} \,A_{\vect{q}^r+\vect{e}_h} \,\overline{\vect{K}}\,\vect{e}_h=:\vect{f}(\vect{e}'(\tau;\vect{e}_h))\\
\vect{e}'(0;\vect{e}_h)&=\vect{e}_h \nonumber
\end{align}
where $\vect{e}_h$ is such that $\vect{q}^r+\vect{e}_h \in \mathcal{B}_{*}(\vect{q}^r)$ and 
$\vect{e}'(\cdot; \cdot) \in \mathbb{R}^m$.

It is worth stressing that, in \cite{Rossi:20}, in presence of square systems, that is, $n=m$, the authors analyzed the evolution of $\dot{\mathbf{q}}(t) = -k\vect{A}^{-1}_{\vect{q}(t)}\vect{A}_{\vect{q}_h}  (\vect{q}_h-\vect{q}^r)$, which is a simplified version of \eqref{eq:sampled_dyn} where $\vect{A}^{\dagger}$ is replaced by $\vect{A}^{-1}$ and where $\bar{\vect{K}}=\vect{I}$.

Let $\vect{e}(t) := \vect{q}(t)-\vect{q}^r$. Note that by direct inspection we have
\begin{align}
\vect{e}(t) =  \vect{e}'(k(t-hT);\vect{e}_h),   \qquad hT \leq t <(h+1)T. \label{eq:relation_q_qP}	 
\end{align} 
Hence, once the solution $\vect{e}'(\tau; \vect{e}_h)$ is computed, then $\vect{e}(t)$ can be obtained through shifting by $hT$ and rescaling by $k$ as long as $\vect{e}'(k(t-hT);\vect{e}_h)$ exists; then, $\vect{q}(t)=\vect{e}(t)+\vect{q}^r$. The major benefit of this approach is that the analysis of \eqref{eqn:NL} is independent of the gain $k$ and the sampling period $T$. Solution of \eqref{eqn:NL} is characterized by interesting and useful properties that we describe in the remaining of this subsection and that we will exploit later on to properly design $k$.
	
For the sake of notational convenience, from now on the symbol $\|\cdot\|$ will be used interchangeably for both $\|\cdot\|_{\vect{D},2}$ and $\|\cdot\|_{\vect{D},\infty}$. 

We start our analysis by observing that $\vect{f}(\vect{e}'(0;\vect{e}_h))=-\overline{\vect{K}}\vect{e}_h$. By exploiting this property it is not difficult to show that the flow of \eqref{eqn:NL} at $\tau=0$ is directed toward the interior of the ball centered at $\vect{q}^r$ and passing through $\vect{e}_h$; this implies that $\|\vect{e}'(0^+;\vect{e}_h)\| < \|\vect{e}_h\|$ being $0^+$ is the time instant just after $\tau=0$.
Based on this fact we can introduce the following temporal variables
\begin{align}
\tau_s(\vect{e}_h) \!&:=\!
\min_{\tau } \{\!\tau \!>\!0  \, | \,  \|\vect{e}'(\tau;\vect{e}_h)\|=\|\vect{e}_h\|\!\} \nonumber\\
\tau_o(\vect{e}_h) \!&:=\! \underset{0 \leq \tau \leq \tau_s(\vect{e}_h)}{\text{arginf}} \|\vect{e}'(\tau;\vect{e}_h)\|, \nonumber
\end{align}
where $\tau_s(\vect{e}_h)=\infty\,$, if $\|\vect{e}'(\tau;\vect{e}_h)\| < \|\vect{e}_h\|$ for all  $\tau>0$. \footnote{A more rigorous and precise analysis of the existence of solution of \eqref{eqn:NL} and of the definition of $\tau_s(\vect{e}_h)$ and $\tau_o(\vect{e}_h)$ can be provided following the lines of Proposition 2 in  \cite{Rossi:20}.}
Basically, $ \tau_s(\vect{e}_h)$ represents the first time that the solution $\vect{e}'(\tau;\vect{e}_h)$ hits the boundary of the ellipse or rectangle depending on the norm used, centered at $\vect{q}^r$ and passing though the initial condition $\vect{e}_h$, 
\ifdraft
\blue{\st{i.e., \emph{stability time}}}
\else 
\fi, while $\tau_o(\vect{e}_h)$,  represents the time that $\vect{e}'(\tau;\vect{e}_h)$ is closest to the origin.

All the above quantities are graphically sketched in Fig.\ref{fig:B0}.
\begin{figure}
	\centering	
	\includegraphics[width=.35\textwidth]{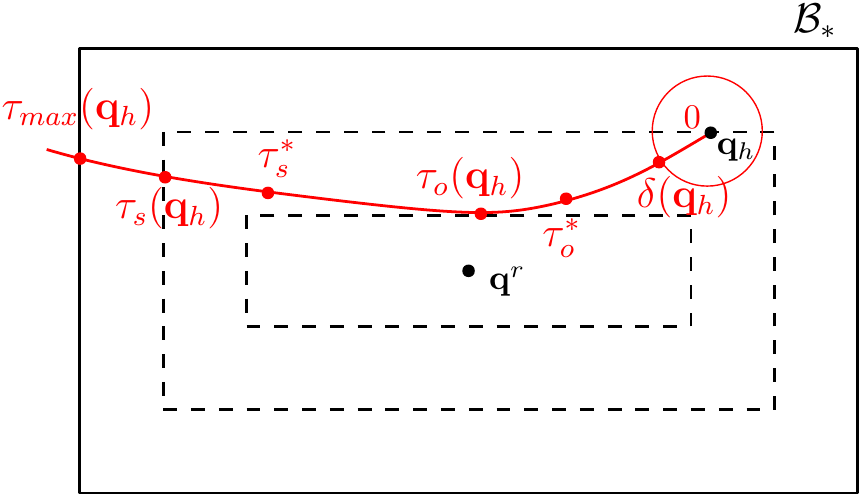}
	\caption{Representation of $\mathcal{B}_0$,  $\delta_{\vect{e}_h}$, $\tau_o^*$, $\tau_o(\vect{e}_h)$, $\tau_s^*$, $\tau_s(\vect{e}_h)$ and $\tau_{max}(\vect{e}_h)$. The red labels represent time values, while the black ones are points or sets in $\mathbb{R}^m$. The $\infty$-norm is considered.}
	\label{fig:B0}
	\vspace{-0.4cm}
\end{figure}
%
The following result provides an important characterization of the trajectories $\vect{e}'(\tau;\vect{e}_h)$.
\begin{prop}\label{prop:tau_0-rho}
	There exist $\bar \tau_o$ and $0 \leq \rho <1$ such that, for all $\vect{e}_h$ s.t. $\vect{q}^r + \vect{e}_h \,\in \,\mathcal{B}_{*}(\vect{q}^r)$, it holds
	\begin{equation*}\label{eq:bar_tau_s}
	0 < \bar \tau_o \leq \tau_o(\vect{e}_h) \leq\infty \,\,\,\,\,\,\,\,\, \text{and} \,\,\,\,\,\,\,\,\, \|\vect{e}'(\bar \tau_o ;\vect{e}_h)\| \leq \rho \| \vect{e}_h\|.
	\end{equation*}
\end{prop}
The proof is constructive and it is provided in the next subsection where we show how to numerically compute a pair $(\bar \tau_o, \rho)$ satisfying these relations. It is worth stressing that $\bar \tau_o$ is independent of $\vect{e}_h$, but, as we will see later on, it is possibly dependent on $\mathcal{B}_{*}(\vect{q}^r)$.

The variables $\bar \tau_o$ and $\tau_{o}(\vect{e}_h)$ allow to propose two different strategies to design the gain $k$. The first is based on the observation that if $k = \frac{\bar \tau_o}{T}$, then $\vect{e}'(\tau;\vect{e}_h) \rightarrow \vect{0} \, \,\,\,\,\, \forall\,\vect{e}_h$ at a convergence rate $\rho$.
In fact, from \eqref{eq:relation_q_qP} we have
$ \|\vect{e}_{h+1}\| =  \|\vect{e}'(kT;\vect{e}_h)\|  =  \|\vect{e}'(\bar \tau_o;\vect{e}_h)\| \leq \rho \|\vect{e}_h\|$. This suggests an \emph{offline procedure} to select $k$ that will be described in the next subsection \ref{subsec:offline_proc}. However, based on the definition of $\tau_o(\vect{e}_h)$, it might be likely that $\|\vect{e}'(\tau_o(\vect{e}_h);\vect{e}_h)\| < \|\vect{e}'(\tau_o^*;\vect{e}_h)\|$ for most $\vect{e}_h$. Therefore, an alternative approach is to select $k$ at each instant $h$ such that $k_h:= \frac{\tau_o(\vect{e}_h)}{T}$. This idea suggests an \emph{online strategy} that will be described in subsection \ref{subsec:online_proc}.
\vspace{-0.1cm}
\subsection{Off-line procedure (Stability and convergence rate)}
\label{subsec:offline_proc}

To exhibit a  pair $(\bar \tau_o, \rho)$ satisfying the relations in Proposition \ref{prop:tau_0-rho}, we start by considering a suitable expansion of the solution of \eqref{eqn:NL}. 
Recall that the solution of \eqref{eqn:NL} can also be written as:
$$ \vect{e}'(\tau;\vect{e}_h)= \vect{e}_h +\int_{0}^\tau \vect{f}(\vect{e}'(\tau';\vect{e}_h)) d\tau', \ \ 0\leq \tau< \tau_{s}(\vect{e}_h).$$
Then,
by using Taylor's theorem for multivariate functions with integral form of the remainder, it becomes	 
\begin{align}  \vect{e}'(\tau;\vect{e}_h)&= \vect{e}_h + \tau\, \vect{f}(\vect{e}'(0;\vect{e}_h)) + \nonumber\\
&+\tau^2 \int_{0}^1 (1-\epsilon) \frac{\partial \vect{f}(\vect{e}'(\epsilon \tau;\vect{e}_h))}{\partial \vect{e}'}\vect{f}(\vect{e}'(\epsilon \tau;\vect{e}_h)) d\epsilon \nonumber \\ 
&= (\vect{I}-\tau\overline{\vect{K}})\vect{e}_h+ \tau^2\vect{d}(\tau,\vect{e}_h) ,   \ \ 0\leq \tau< \tau_{s}(\vect{e}_h), \label{eq:reminder1}
\end{align}
where the reminder $\vect{d}$ is defined by the last equality. 
%
As a consequence of the properties \eqref{eqn:upper_bounds1} and \eqref{eqn:upper_bounds2}, we have that	 
	\begin{align*}
\| \vect{d}(\tau,\vect{e}_h)	\|&\leq\int_{0}^1 (1-\epsilon) \bigg\|\frac{\partial \vect{f}(\vect{e}'(\epsilon \tau);\vect{e}_h)}{\partial \vect{e}'}\bigg\| \| \vect{f}(\vect{e}'(\epsilon \tau);\vect{e}_h)\| d\epsilon \\
	&\leq \int_{0}^1 (1-\epsilon)ab \|\vect{e}_h\|^2 d\epsilon = \frac{1}{2}ab\|\vect{e}_h\|^2 
	\end{align*}
	where we have exploited the properties that, for all $\vect{e}$, $\vect{e}'$ such that $\vect{q}^r+\vect{e}$, $\vect{q}^r+\vect{e}'$ belong to $\mathcal{B}_{*}(\vect{q}^r)$,
\begin{eqnarray}
\| \vect{f}(\vect{e}'(\tau);\vect{e})\| \hspace{-0.2cm} &\leq& \hspace{-0.2cm} \| A^{\dagger}_{\vect{q}^r+\vect{e}'}\,A_{\vect{q}^r+\vect{e}}\,\overline{\vect{K}}\|\|\vect{e}\| = a \|\vect{e}\|, \   \ \ \ \  \label{eqn:upper_bounds1}\\
\left\| \frac{\partial \vect{f}(\vect{e}'(\tau);\vect{e})}{\partial \vect{e}'}\right\| \hspace{-0.2cm} &\leq& \hspace{-0.2cm} b \|\vect{e}\|\,
\label{eqn:upper_bounds2}
\end{eqnarray} 
for some $a,b>0$ since $\vect{f}$ and $\frac{\partial \vect{f}}{ \partial \vect{e}'}$  are continuous maps on a compact domain.
	
Now, notice that since $\vect{e}_h$ is such that $\vect{q}^r+ \vect{e}_h \in  \mathcal{B}_{*}(\vect{q}^r)$, then $\|\vect{e}_h\| \leq R$ and, hence, by defining,
\begin{align}\label{eq:mu_def}
\mu:=\frac{1}{2}abR,
\end{align}
we can write $\| \vect{d}(\tau,\vect{e}_h)\| \leq \mu \|\vect{e}_h\|$. 
From \eqref{eq:reminder1} and the upper bound on $\vect{d}$, it follows
\vspace{0cm}
\begin{align}
\| \vect{e}'(\tau;\vect{e}_h) \| &\leq \|(\vect{I}-\tau\overline{\vect{K}})\vect{e}_h\|+\mu \tau^2\|\vect{e}_h\| \nonumber\\
& \leq \left( \|(\vect{I}-\tau\overline{\vect{K}})\|+\mu \tau^2 \right)\,    \|\vect{e}_h\|
\label{eq:first_def_of_g}
\end{align}
for all $0\leq \tau< \tau_{s}(\vect{e}_h)$, where $\|(\vect{I}-\tau\overline{\vect{K}})\|$ denotes the matrix norm induced by the corresponding vector norm. In order to evaluate upper bounds for 
\ifdraft
\blue{\st{asymptotic stability and}}
\else 
\fi the convergence rate, we need to study the following function 
\vspace{0cm}
$$ g(\tau;\mu):= \|\vect{I}-\tau\overline{\vect{K}}\|+\mu\tau^2 $$ 
Since $\overline{\vect{K}}$ is diagonal we can write
\begin{align}\label{g_tau_mu}
g(\tau;\mu)
&= \max_i\frac{|1-\tau \overline{k}_i|}{d_i^2}+\mu\tau^2,
\end{align}
independently from the norm used.


\begin{figure}
	\centering	
	\includegraphics[width=.4\textwidth]{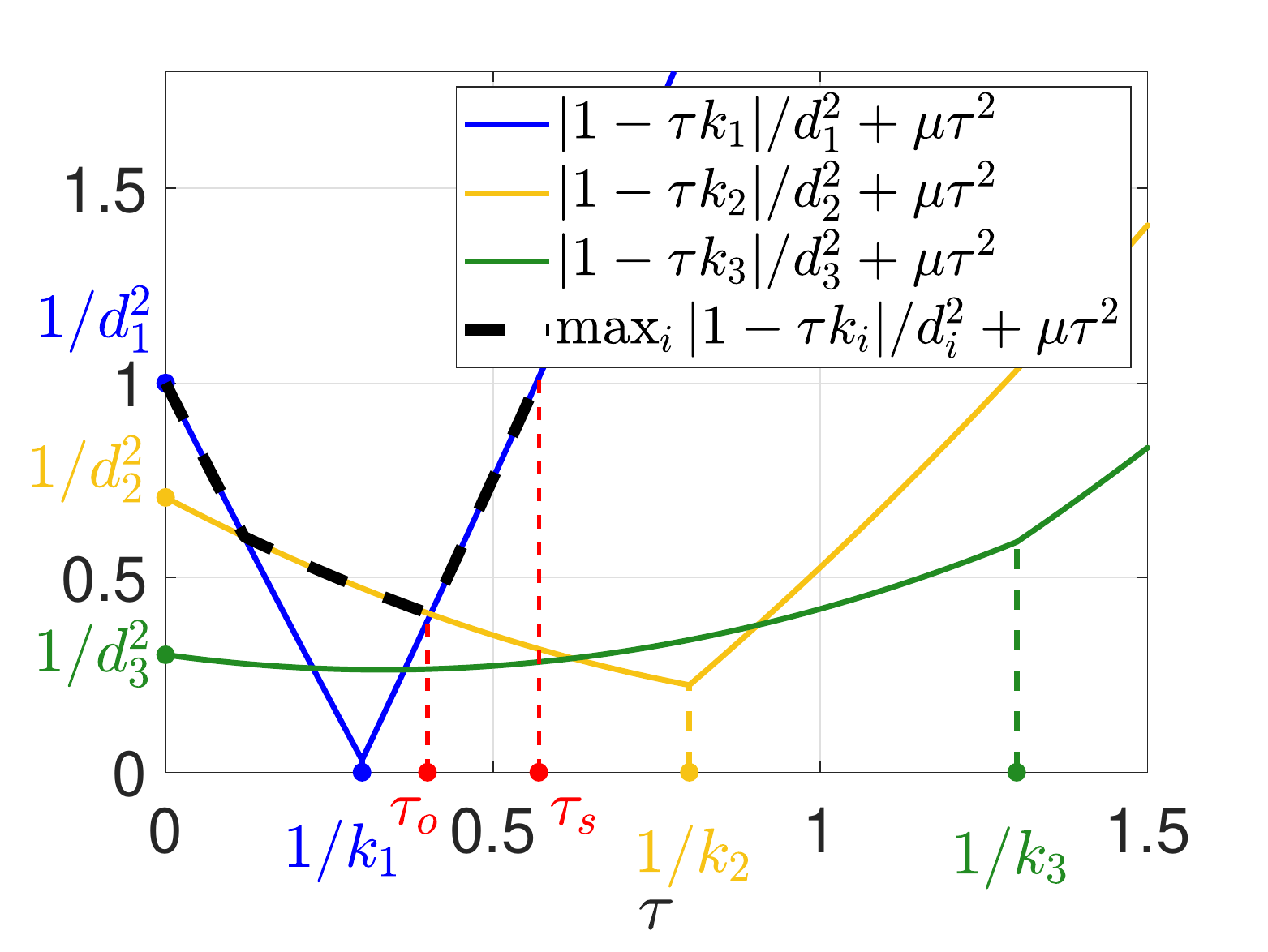}
	\caption{Representation of $\frac{|1-\tau \overline{k}_i|}{d_i^2}+\mu \tau^2$ and of the maximum.}
	\label{fig:g_ext_2}
\end{figure}

Observe that $g(\tau;\mu)$ is a strictly convex function and that,
from \eqref{eq:d_i}, it follows  
$
g(0;\mu)=\max_i\frac{1}{d_i^2} \leq 1.
$
Moreover, $
g'(0;\mu) < 0$, thus implying that $\min_{\tau\geq0} \,\, g(\tau;\mu)<1$. Based on the above properties of the function $g$ we can define
$$
\overline\tau_s(\mu) := \left\{\tau >0 \,|\, g(\tau;\mu)=1\right\},
$$
$$ 
\overline{\tau}_o(\mu) := \underset{0\leq \tau \leq \overline\tau_s(\mu)}{\text{argmin}}\,\,\,  g(\tau;\mu).
$$
A pictorial description of function $g$ and quantities $\overline\tau_s(\mu)$ and $\overline{\tau}_o(\mu)$ is provided in Figure \ref{fig:g_ext_2}.
Accordingly let 
$$
\rho(\mu)=g(\overline{\tau}_o(\mu);\mu).
$$
Since $\rho(\mu)<1$, it turns out that the pair $\left(\overline{\tau}_o(\mu), \rho(\mu)\right)$ satisfies the properties stated in Proposition \ref{prop:tau_0-rho}.

The strict-convexity property of the function $g$ allows to compute $\overline\tau_s(\mu)$, $\overline{\tau}_o(\mu)$ and $\rho(\mu)$ in a efficient way, by resorting to standard numerical toolboxes. 

The previous reasonings suggest that a possible choice for the gain $k$, once the sampling time $T$ is known, is given by
	\vspace{0cm}
	\begin{equation}\label{eq:k*}
	k^* = \frac{\overline{\tau}_o(\mu)}{T},
	\end{equation}
	as formally established in the next proposition.
	\begin{prop}\label{prop:asympt_stable}
		For all $ \vect{e}(0)$ such that $ \vect{e}(0)+ \vect{q}^r \in \mathcal{B}_{*}(\vect{q}^r)$, the following inequality holds
		$$\|{\vect{e}(hT)}\|\leq \rho^h(\mu)\|{\vect{e}(0)}\|
		$$
		and $\|{\vect{e}(t)}\|\leq \|{\vect{e}(hT)}\|$ for all $hT \leq t < (h+1)T$.
	\end{prop}
	Notice that, from Prop.\ref{prop:asympt_stable}, it turns out that the origin is an asymptotically stable equilibrium for the system and the proposed \emph{offline} strategy converges exponentially fast with a rate at least $\rho(\mu)$.	We conclude this Section observing that, since $\mu$ can be computed a-priori before running the algorithm, then the offline strategy is amenable of the sparse implementation \eqref{eqn:sparse}.

\begin{remark}[\bf A numerical refinement of the upper bound on $\vect{d}$]\label{rem:MC-mu}
Note that $\mu$, as defined in \eqref{eq:mu_def}, represents a rough estimate of the upper bound of $\| \vect{d}(\tau,\vect{e}_h)\| $. However this estimate can be refined as follows. Let
\vspace{0cm}
$$\vect{d}'(\tau, \vect{e}_h):=\tau^2\vect{d}(\tau, \vect{e}_h)\overset{\eqref{eq:reminder1}}{=}\vect{e}'(\tau, \vect{e}_h)-(\vect{I}-\tau\overline{\vect{K}})\vect{e}_h,$$
and
\vspace{-0.3cm}
\begin{align}
\mu^*& :=\inf_{\gamma}\{ \gamma \,\, |\,\, \|\vect{d}'(\tau, \vect{e}_h)\| \leq \gamma \|\vect{e}_h\|\tau^2, \, \label{eq:mu_star}\\
&\qquad \qquad \qquad \forall \, \vect{e}_h, \,\,\, \forall \tau \in (0,\tau_s(\vect{e}_h))\}\nonumber.
\end{align}
Notice that $\mu^* \leq \mu$.
An estimate $\hat{\mu}^*$ of $\mu^*$ can be computed adopting a Monte-Carlo sampling approach which consists in calculating the right-hand side of \eqref{eq:mu_star} over a finite number of trajectories starting from initial conditions randomly sampled within $\mathcal{B}_{*}(\vect{q})$. Clearly $\hat{\mu}^* \rightarrow \mu^*$ as the number of initial samples increases. 

%

If $\hat{\mu}^*$ has been computed using a sufficiently high number of initial samples, we could replace $\mu$ with $\hat{\mu}^*$ in the expression of function $g$ thus obtaining a pair $\left(\overline{\tau}_o\left(\hat{\mu}^*\right), \rho\left(\hat{\mu}^*\right)\right)$ which should lead to better performance. In particular, it is easy to see that, since $\hat{\mu}^*\leq \mu$, then $\rho\left(\hat{\mu}^*\right) \leq \rho\left(\mu\right)$.
\end{remark}


		\subsection{Online model-predictive procedure}	
	\label{subsec:online_proc}
	In this section we assume that the future trajectory $\vect{e}'(\tau, \vect{e}_h)$ can be numerically computed based on the model dynamics $\vect{f}(\vect{q};\vect{e}_h)$ and the current position $\vect{e}_h$. This implies that also $\tau_o(\vect{e}_h)$ can be computed at any time step $h$. If so, we can propose the following input
	$$ \vect{u}(t) = \vect{u}_h =  -k_h \vect{A}_{\vect{e}_h} \overline{\vect{K}}\vect{e}_h ,  \qquad \qquad hT \leq t <(h+1)T, $$
	where
	\begin{equation}\label{eq:k_h}
	k_h :=  \frac{\tau_o(\vect{e}_h)}{T}.
	\end{equation}
	A more precise characterization of the convergence properties of this strategy is stated in the next proposition.
	\begin{prop}\label{prop:quadr_conv}
		Consider the system in \eqref{eq:sampled_dyn} with a time varying sequence of gains $k_0, k_1, k_2, \ldots$,  where the generic $k_h$ is given as in \eqref{eq:k_h}. Then we have that
		$$\|\vect{e}(t)\|\leq \|\vect{e}(hT)\|$$
		for all $hT \leq t < (h+1)T$, and 
		$$
		\|\vect{e}(hT)\| \leq\, \rho^h(\mu)\, \| \vect{e}(0)\|.
		$$
		Remarkably, the convergence is at least quadratic if one of the following two facts is verified
		\begin{itemize}
			\item the weights of the matrix $\vect{D}$ are all equal with each other, that is, 
			$$d_1= d_2=\ldots= d_N=d_L;$$
			\item the gains defining the matrix $\overline{\vect{K}}$ are all equal with each other, that is,
			$$k_1= k_2=\ldots= k_N=k_L.$$
		\end{itemize} 
	\end{prop} 
	\begin{proof}
		Observe that, according to \eqref{eq:k_h}, we necessarily have:
		\begin{align*}
		\|\vect{e}_{h+1}\| = \|\vect{e}'(\tau_o(\vect{e}_h);\vect{e}_h)\| 
		\leq \|\vect{e}'(\overline{\tau}_o;\vect{e}_h)\| \leq \rho \|\vect{e}_h\|,
		\end{align*}
		hence the proposed scheme is exponentially stable with rate $\rho$ for any $T$. This proves the first part of the Proposition. 
		
		Consider now the case where all the $d_i$ are equal to the same value $\bar{d}$. In this case we have
		that \eqref{g_tau_mu} can be rewritten as
		\begin{align*}
g(\tau;\mu)
&= \max_i\frac{|1-\tau \overline{k}_i|}{\bar{d}}+\mu\tau^2= \frac{|1-\tau \overline{k}_{min}|}{\bar{d}}+\mu\tau^2
\end{align*}
where $\overline{k}_{min} = \min_{i} k_i$.

Since in the online scenario $r=\|\vect{e}_h\| \rightarrow 0$, then $\mu= \frac{1}{2}ab \|\vect{e}_h\| \rightarrow 0$. This implies that $\overline{\tau}_o(\mu) \to 1/ \bar{k}_{min}$ and, in turn, $\rho(\mu) \to \frac{1}{\bar{k}_{min}^2} \mu$. As so, $\|\vect{e}_{h+1}\|\leq \rho\|\vect{e}_h\| \to \frac{1}{2 \bar{k}_{min}^2}ab\|\vect{e}_h\|^2$.

		As a consequence		 
		\begin{equation*}
		\limsup_{h \rightarrow +\infty}{\frac{\|\vect{e}_{h+1}\|}{\|\vect{e}_h\|^2}}\leq\frac{1}{2  \bar{k}_{min}^2}ab, 
		\end{equation*}
		and $\|\vect{e}_{h}\|\leq (\frac{1}{2  \bar{k}_{min}^2}ab\|\vect{e}_0\|)^{(2^h-1)}\|\vect{e}_0\|$.
		Since $ab>0$, then the quadratic convergence of the sequence $\|\vect{e}_h\|$ is guaranteed. 
		
		Similar reasonings hold for the scenario where the gains defining the matrix $\overline{\vect{K}}$ are all equal with each other.
	\end{proof}
	Based on the definition of $\tau_o(\vect{e}_h)$ and on Prop.~\ref{prop:quadr_conv}, we expect the online strategy to exhibit a faster convergence than the offline one. This fact is supported also by the numerical results reported in the next section. 
	However, the higher rate of convergence comes at the price of a heavier computational load. Indeed $\tau_o(\vect{e}_h)$ needs to be estimated at each iteration and a global knowledge of the vector $\vect{e}_h$ is required. In particular the online strategy exhibits the structure of the control architecture in \eqref{eqn:centralized}; this implies that the online strategy cannot be implemented in a decentralized way, but only in a centralized fashion.

	\section{Simulation Results}\label{sec:simul}
In the first part of this section we discuss the results obtained by simulating the off-line and on-line techniques applied to the kinematic model of the system depicted in Figure \ref{fig:squared_flycrane}. In detail, we assume that the sampling time is $T=1.5 \ [s]$ and the gain and weight matrices are, respectively, $\bar{\vect{K}}=\vect{I}_{10}$ and $\vect{D}=\text{diag}\{\vect{I}_4, \ 4 \cdot \vect{I}_6\}$. In this case we chose to differently weight the variables $q_i, \, i=1,2,3,4$ with respect to $\vect{q}_L$. 
\begin{figure}[h!]
	\centering	
	\includegraphics[width=0.5\textwidth]{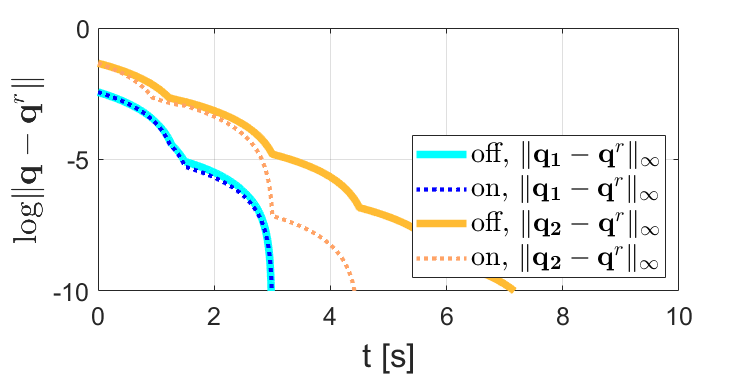}
	\caption{Evolution of the error norm $\|\vect{q}-\vect{q}^r\|_\infty$ for different initial conditions ($\vect{q}_1$, $\vect{q}_2$) and techniques (on-line and off-line) when the robots velocities generated by the planner are applied to the kinematic system. In this case
		$\vect{K}=\vect{I}_{10}$, $\vect{D}=\text{diag}\{\vect{I}_4, 4\cdot \vect{I}_6\}$.}
	\label{fig:ON_OFF_Inf_norm_diff_D_K_eye_diff_q0}
\end{figure}
In Figure \ref{fig:ON_OFF_Inf_norm_diff_D_K_eye_diff_q0} we compare the behavior of the system when the norm $\| \cdot \|_{\vect{D},\infty}$ is considered and the off-line and on-line strategies are compared. In addition, two different initial conditions $\vect{q}_1$ and $\vect{q}_2$ are taken into account, where $\vect{q}_1$ is significantly closer to $\vect{q}^r$ than $\vect{q}_2$. To compute $\mu$ we assume that, when $\vect{q}(0)=\vect{q}_1$ then $\mathcal{B}_*(\vect{q}^r)$ is determined by $R= \|\vect{q}_1\|$, while when $\vect{q}(0)=\vect{q}_2$ then $R= \|\vect{q}_2\|$. Moreover the parameter $\mu$ is evaluated according to the procedure described in Remark \ref{rem:MC-mu}.
Figure \ref{fig:ON_OFF_Inf_norm_diff_D_K_eye_diff_q0} shows that, when $\vect{q}(0)=\vect{q}_1$, then the two strategies have more or less the same performance and, remarkably, they exhibit a convergence that is faster than linear. 
Instead, the difference between the two strategies is evident when $\vect{q}(0)=\vect{q}_2$. This is due to the fact that the value of $\mu$ is significantly smaller when $\vect{q}(0)=\vect{q}_1$ than when $\vect{q}(0)=\vect{q}_2$; as a consequence, also the minimum of $g$ is smaller and the convergence rate is closer to the one of the on-line scenario. Similar results are obtained using the weighted 2-norm and are not report here in the interest of space.

Finally, in Figure \ref{fig:q_err_norm_varying_alpha} we plotted the results obtained by simulating the  dynamical model of the system \eqref{eq:joint_dyn_eq} when the sampled control law in \eqref{eq:SampledControlLaw} is applied: the goal is to show numerically the validity of Theorem \ref{thm:sing_pert_sys_adapted} also in presence of sampled measurements. Here we only consider the $2$-norm, indeed the $\infty$-norm would give the same result. Consistently with the result of Theorem \ref{thm:sing_pert_sys_adapted}, the picture highlights that for large values of $\alpha$ the convergence of the system is achieved, while for lower values the norm of the error is not guaranteed to converge to zero. 
\begin{figure}
	\centering	
	\includegraphics[width=0.5\textwidth]{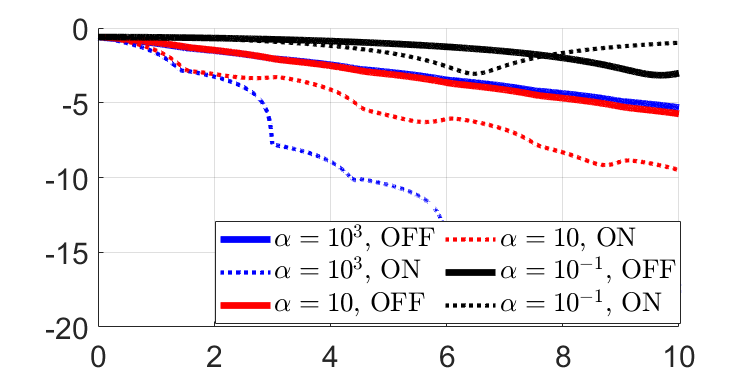}
	\caption{Evolution of the error norm $\|\vect{q}-\vect{q}^r\|_2$ when the on-line and off-line techniques are applied: moreover, three different values of the gain $\alpha$ are chosen.}
	\label{fig:q_err_norm_varying_alpha}
\end{figure}
	

	\section{Conclusions}\label{sec:concl}
		\vspace{-0.1cm}


	\appendices
	\section{Theorem 11.4 of \cite{khalil2002nonlinear}}
	\label{app:stand_sing_pert_model}
	\begin{thm}\label{thm:Khalil}
		Consider the singularity perturbed system
		\begin{align}
		\dot{x}&=f(t,x,z,\varepsilon)\label{eq:sing_pert_sys1}\\
		\varepsilon\dot{z}&=g(t,x,z,\varepsilon)
		\label{eq:sing_pert_sys2}
		\end{align}
		Assume that the following assumptions are satisfied for all
		$$(t,x,z,\varepsilon)\in [0,\infty) \, \times \, B_r \, \times \, [0,\varepsilon_0]$$
		\begin{enumerate}
			\item $f(t,0,0,\varepsilon)=0$ and $g(t,0,0,\varepsilon)=0$.
			\item The equation
			$$0=g(t,x,z,0)$$
			has an isolated root $z=h(t,x)$ such that $h(t,0)=0$.
			\item  The functions $f,g,h$ and their partial derivatives up to the second order are bounded for $y=z-h(t,x)$.
			\item The origin of the reduced system 
			$$\dot{x}=f(t,x,h(t,x),0)$$
			is exponentially stable.
			\item The origin of the boundary-layer system
			$$\frac{dy}{d\tau}=g(t,x,y+h(t,x),0)$$
			is exponentially stable, uniformly in $(t,x)$.
		\end{enumerate}
	Then, there exists $\varepsilon^*>0$ such that for all $\varepsilon<\varepsilon^*$, the origin of \eqref{eq:sing_pert_sys1}-\eqref{eq:sing_pert_sys2} is exponentially stable.
	\end{thm}

\end{document}